\documentclass[11pt]{article}
\usepackage{color}
\usepackage[margin=1in]{geometry}
\usepackage{amssymb}
\usepackage{amsmath}
\usepackage{mathrsfs}
\usepackage{amsthm}
\usepackage{yhmath}
\usepackage{graphicx}
\usepackage{cases}
\usepackage{cite}
\usepackage{authblk}

\newtheorem{theorem}{Theorem}

\newtheorem{definition}{Definition}

\newtheorem{assumption}{Assumption}
\newcommand{\R}{\mathbb{R}}

\newcommand{\norm}[1]{\left\lVert#1\right\rVert}
\newcommand{\itp}{\hat{\otimes}_\varepsilon}

\title{Learning solution operators of PDEs defined on varying domains via MIONet}

\author[1,2]{Shanshan Xiao}
\author[3,*]{Pengzhan Jin}
\author[1,2]{Yifa Tang}

\affil[1]{LSEC, ICMSEC, Academy of Mathematics and Systems Science, Chinese Academy of Sciences, Beijing 100190,
China}
\affil[2]{School of Mathematical Sciences, University of Chinese Academy of Sciences, Beijing 100049, China}
\affil[3]{School of Mathematical Sciences, Peking University, Beijing 100871, China}
\affil[*]{Corresponding author. E-mail: jpz@pku.edu.cn}

\date{}

\begin{document}
	
\maketitle
\begin{abstract}
    In this work, we propose a method to learn the solution operators of PDEs defined on varying domains via MIONet, and theoretically justify this method. We first extend the approximation theory of MIONet to further deal with metric spaces, establishing that MIONet can approximate mappings with multiple inputs in metric spaces. Subsequently, we construct a set consisting of some appropriate regions and provide a metric on this set thus make it a metric space, which satisfies the approximation condition of MIONet. Building upon the theoretical foundation, we are able to learn the solution mapping of a PDE with all the parameters varying, including the parameters of the differential operator, the right-hand side term, the boundary condition, as well as the domain. Without loss of generality, we for example perform the experiments for 2-d Poisson equations, where the domains and the right-hand side terms are varying. The results provide insights into the performance of this method across convex polygons, polar regions with smooth boundary, and predictions for different levels of discretization on one task. We also show the additional result of the fully-parameterized case in the appendix for interested readers. Reasonably, we point out that this is a meshless method, hence can be flexibly used as a general solver for a type of PDE.
\end{abstract}

\section{Introduction}

In recent years, scientific machine learning (SciML) has achieved remarkable success in computational science and engineering \cite{karniadakis2021physics}. Due to the powerful approximation ability of neural networks (NNs) \cite{e2019barron,hanin2019universal,hornik1989multilayer,hornik1990universal,siegel2022high}, different methods are proposed to solve PDEs by parameterizing the solutions via NNs with the loss functions constructed by strong/variation forms of PDEs, such as the PINNs \cite{cai2021physics,lu2021physics,pang2019fpinns,raissi2019physics}, the deep Ritz method \cite{yu2018deep}, and the deep Galerkin method \cite{sirignano2018dgm}. A drawback of these methods is their slow solving speed, as one has to re-train the NN as long as the parameters of the PDE are changed. To fast obtain the solutions of parametric PDEs, several end-to-end methods called neural operators, are proposed to directly learn the solution operators of PDEs. The DeepONet \cite{lu2019deeponet,lu2021learning} was firstly proposed in 2019, which employs a branch net and a trunk net to encode the input function and the solution respectively, achieving fast prediction for parametric PDEs. In 2020, the FNO \cite{li2020neural,li2020fourier} was proposed to learn the solution mappings with the integral kernel parameterized in Fourier space. For the same topic, there are lots of works \cite{gupta2021multiwavelet,he2023mgno,jin2022mionet,rahman2022u,raonic2023convolutional,wu2023solving} developing the field of neural operators. Among the neural operators, the MIONet \cite{jin2022mionet} plays an important role in dealing with complicated cases with multiple inputs, which generalizes the theory and the architecture of the DeepONet. The DeepONet and the MIONet are in fact both derived from the tensor product of Banach spaces, which also leads to tensor-based machine learning models for eigenvalue problems \cite{hu2023experimental,wang2022tensor}. Moreover, the trunk nets of ONets series provide the convenience of differentiation for the output functions, hence training a neural operator without data is being possible via utilizing the PDEs information, which is studied as physics-informed DeepONet/MIONet \cite{wang2021learning,zheng2023state}.

One limitation of these methods lies in their exclusive treatment of PDEs with fixed domains. As many real-world PDE problems involve diverse domains, there is a need to develop the capability of neural operators for PDE problems with varying regions. In the realm of neural operators, limited researches have been dedicated to address the issue of varying domains, in which there are fundamental difficulties. Recently, \cite{goswami2022deep} employs the DeepONet to adapt PDEs with different geometric domains. They enable the transformation of results from one domain to another via the transfer learning. However, such treatment still requires a re-training process for the model once the domain of the PDE is changed. Another work related to PDEs on varying domains is Geo-FNO \cite{li2022fourier}, an extension of FNO. Geo-FNO enhances the versatility in managing arbitrary domains by transforming the physical space into a regular computational space through diffeomorphism. The primary approach involves using diffeomorphisms to convert meshes in physical space into uniform meshes. Geo-FNO is capable of learning the mapping from the parameterized domain, the initial condition, and the boundary condition to the corresponding solution. Note that Geo-FNO is limited to addressing problems characterized by identical PDE forms but varying domains.

Recall that leveraging MIONet allows us to acquire the solution operators of PDEs with additional inputs, not only the initial and the boundary conditions, but also the parameters in PDEs. In this work, we present a method built upon MIONet, enabling the learning of solution operators for PDEs defined on varying regions, especially, it allows not only the regions and the initial/boundary conditions but also the parameters in PDEs to be changing. Consequently, the method can predict solutions for fully-parameterized PDEs. Our primary approach is to conceptualize the disjoint union of infinite regions as a metric space, and then extend the theory of MIONet from Banach spaces to metric spaces. Building upon these theoretical foundations, we initially define the metric space $U$ comprised of polar regions and confirm that $U$ satisfies the projection assumption necessary for the approximation condition of MIONet. Subsequently, we project the input function space
\begin{equation}
X = \bigsqcup_{\Omega\in U}C (\overline{\Omega})
\end{equation}
onto the Cartesian product of the metric space $U$ and the Banach space $C(B(0,1))$. Through this transformation, we acquire an equivalent solution mapping 
\begin{equation}
\hat{\mathcal{G}}: U \times C(B(0,1))\rightarrow C(B(0,1)),
\end{equation}
which can be learned by MIONet under our new theory. We are then capable of predicting solutions for fully-parameterized PDEs. We outline the principal contributions of our work as follows:
\begin{itemize}
	\item We extend the theory of MIONet to deal with metric spaces.
    \item We construct a space comprised of appropriate regions and establish a well-defined metric on this space. We then prove that this metric space satisfies the approximation condition of MIONet.
    \item We propose a MIONet-based algorithm tailored for PDEs defined on varying domains, which is able to directly predict solutions for fully-parameterized PDEs.
\end{itemize}

The structure of this paper is as follows. We first introduce the research problem in Section \ref{sec:setup}. In Section \ref{sec:theory}, we establish the theoretical foundation of our work, and subsequently propose the method in detail. Section \ref{sec:experiment} presents the results of numerical experiments, where we evaluate our method's performance on 2-d Poisson equations over several different types of regions. Finally, we summarize our findings and contributions in Section \ref{sec:conclusions}.

\section{Problem setup}\label{sec:setup}

This research originates from the limitations of the current neural operators which mainly learn the solution operators of PDEs defined on fixed domains. In practical scenarios, PDE problems often involve varying domains.

The mapping from the input functions to the corresponding solution of a PDE defined on varying domains can be written as
\begin{equation}
    f^{1}_{\Omega_1}\times f^{2}_{\Omega_2} \times \cdots \times f^{m}_{\Omega_m} \mapsto u_{\Omega},
\end{equation}
where $\Omega_i$ are the domains that do not need to be the same as $\Omega$. $f^{i}_{\Omega_i}$ and $u_{\Omega}$ are functions defined on $\Omega_i$ and $\Omega$, respectively. 
Taking the Poisson equation as an example:
\begin{equation}\label{eq:full_problem_1}
	\begin{cases}
		-\nabla\cdot(k\nabla u)=f &\quad \mbox{in} \; \Omega,
		\\
		u = g &\quad \mbox{on} \; \partial\Omega,
	\end{cases}	
\end{equation}
then the solution mapping is
\begin{equation}\label{eq:full_problem_2}
    k_{\Omega}\times f_{\Omega} \times g_{\partial\Omega} \mapsto u_{\Omega}.
\end{equation}
Note that here $\Omega$ is not fixed, i.e., different tasks provide different $\Omega$. Now we consider the space $U$ that consists of some domains $\Omega$ in $\mathbb{R}^{d}$, and $X$ consists of the functions defined on domains in $U$, then $X$ will not be a Banach space, so that we cannot directly employ the neural operators to learn this end-to-end map. However, such a space $X$ could be equipped with an appropriate metric that deduces some necessary properties for operator regression. 

Since the difficulty lies in dealing with the varying domains, in order to facilitate the readers to understand, here we consider the simplified case
\begin{equation}
	\begin{cases}
		-\Delta u=f &\quad \mbox{in} \; \Omega,
		\\
		u = 0 &\quad \mbox{on} \; \partial\Omega,
	\end{cases}	
\end{equation}
with the solution mapping
\begin{equation}\label{eq:poisson_mapping}
    f_{\Omega}\mapsto u_{\Omega}.
\end{equation}
The case of (\ref{eq:full_problem_1}-\ref{eq:full_problem_2}) will be of no difficulty as long as this simplified case is solved. Readers can refer to Appendix \ref{app:full} for details of the fully-parameterized case.

Next we have to further develop the theory and the method for learning such mappings based on current neural operators.

\section{Theory and method}\label{sec:theory}

\subsection{MIONet for metric spaces}
In this section, we aim to extend the theory of MIONet \cite{jin2022mionet} from Banach spaces to metric spaces, i.e., to demonstrate that MIONet can deal with continuous mappings defined on metric spaces that satisfy specific conditions. Firstly, we give an assumption on the metric spaces.
\begin{assumption}[projection assumption]\label{ass:projection}
Let $X$ be a metric space with metric $d(\cdot,\cdot)$, assume that $\{\phi_n\}$ and $\{\psi_n\}$ are two sets of mappings, with $\phi_n\in C(X, \R^n)$, $\psi_n\in C(\phi_n(X), X)$, $P_n:=\psi_n\circ\phi_n\in C(X, X)$, satisfying
\begin{equation}
    \lim_{n\to\infty}\sup_{x\in K}d(x,P_n(x))=0,
\end{equation}
for any compact $K\subset X$. We say $\{\phi_n\}$ is a discretization for $X$, and $\{\psi_n\}$ is a reconstruction for $X$, $P_n$ is the corresponding projection mapping. Note that the space of the image of $\phi_n$ does not necessarily need to be $n$-dimensional, any fixed integer positively related to $n$ is permitted and will not affect the property.
\end{assumption}
\noindent The assumption is proposed to substitute for the approximation property of Schauder basis for Banach space in a weak setting.
\begin{theorem}[approximation theory]\label{thm:approximation}
    Let $X_{i}$ be metric spaces and $Y$ be a Banach space, assume that $X_{i}$ satisfies Assumption \ref{ass:projection} with the projection mapping $P_q^i=\psi_q^i\circ\phi_q^i$, $K_{i}$ is a compact set in $X_{i}$. Suppose that
    \begin{equation}
    \mathcal{G}: K_{1}\times \cdots \times K_{n} \rightarrow Y   
    \end{equation}
    is a continuous mapping, then for any $\epsilon>0$, there exist positive integers $p_{i}, q_{i}$, continuous functions $g^{i}_{j}\in C(\mathbb{R}^{q_{i}})$ and $u_{j} \in Y$ such that 
    \begin{equation}\label{eq:appr_theory}
    \sup_{v_{i}\in K_{i}} \norm{\mathcal{G}(v_{1},\cdots,v_{n}) - \sum_{j = 1}^{p}g_{j}^{1}(\phi_{q_{1}}^{1}(v_{1}))\cdots g_{j}^{n}(\phi_{q_{n}}^{n}(v_{n}))\cdot u_{j}}_Y< \epsilon.
    \end{equation}
\end{theorem}
\begin{proof}
Note that the proof of the approximation theory of MIONet for Banach spaces in fact only utilizes the approximation result of Schauder basis, and it does not involve other properties of the Banach spaces. Such an approximation result can be replaced by Assumption \ref{ass:projection}. Here we simply show the key points.

Based on the injective tensor product
\begin{equation}\label{eq:itp}
    C(K_1\times K_2\times\cdots\times K_n,Y)\cong C(K_1)\itp C(K_2)\itp\cdots\itp C(K_n)\itp Y,
\end{equation}
we obtain
\begin{equation}\label{eq:itp_appr}
\norm{\mathcal{G}-\sum_{j=1}^p f_{j}^1\cdot f_{j}^2\cdots f_{j}^n\cdot u_j}_{C(K_1\times K_2\times\cdots\times K_n,Y)}<\epsilon
\end{equation}
for some $f_j^i\in C(K_i)$ and $u_j\in Y$. Assumption \ref{ass:projection} shows that there exist sufficiently large $q_i$, such that
\begin{equation}
\norm{\mathcal{G}-\sum_{j=1}^p f_{j}^1(P^1_{q_1}(\cdot))\cdot f_{j}^2(P^2_{q_2}(\cdot))\cdots f_{j}^n(P^n_{q_n}(\cdot))\cdot u_j}_{C(K_1\times K_2\times\cdots\times K_n,Y)}<\epsilon.
\end{equation}
Denote $g^i_j:=f^i_j\circ\psi^i_{q_i}$, then we immediately obtain \eqref{eq:appr_theory}.
\end{proof}

\subsection{Setting for varying domains}
In this work, we discuss a special type of regions in $\R^2$. The strategy can also be applied to higher dimensional case with similar treatment.
\begin{definition}\label{def:polar_region}
    We define \textbf{polar regions} as follows: Consider the centroid of an open region $\Omega\subset\R^2$ as the original point, then we refer to $\Omega$ as a polar region if its boundary $\partial\Omega$ can be expressed as a Lipschitz continuous function with a period of $2\pi$ under polar coordinates.
\end{definition}
\noindent Now we considering two spaces. Denote
\begin{equation}
    U:=\{ \Omega\subset \R^2 ~|~ \Omega ~ \textrm{is a polar region with a Lipschitz coefficient defined above no more than $L$}\}
\end{equation}
and
\begin{equation}
    X:=\bigsqcup_{\Omega\in U}C\left(\overline{\Omega}\right).
\end{equation}
We will provide metrics on these two spaces, thus make $U$ and $X$ two metric spaces.

Firstly, we define a transformation $\alpha_{\Omega}$ that maps the closed polar region $\overline{\Omega}$ onto the closed unit ball $B(0,1)$. Assume that $O'$ and $O$ are the centroids of $\Omega$ and $B(0,1)$ respectively. For any point $p\in\overline{\Omega}$, let $\hat{p}$ be the unique intersection point in $\{O'+t(p-O')|t\geq0\}\cap\partial\Omega$. We denote the angle anticlockwise from $e_0:=(1,0)$ to $p-O'$ as $\theta$. Now we map $\overline{\Omega}$ onto $B(0,1)$ as
\begin{equation}
\begin{split}
    \alpha_{\Omega}: \overline{\Omega} &\rightarrow B(0,1) \\
         p &\mapsto \frac{|p-O'|}{|\hat{p}-O'|}(\cos(\theta),\sin(\theta)).
\end{split}
\end{equation}
Clearly, $\alpha_\Omega$ is a bijection from $\overline{\Omega}$ to $B(0,1)$, and there exists an inverse mapping $\alpha_\Omega^{-1}$. Both $\alpha_\Omega$ and $\alpha_\Omega^{-1}$ are continuous. An illustration is shown in Figure \ref{fig:transformation}.

\begin{figure}[htbp]
\centering
\includegraphics[width=0.65\textwidth]{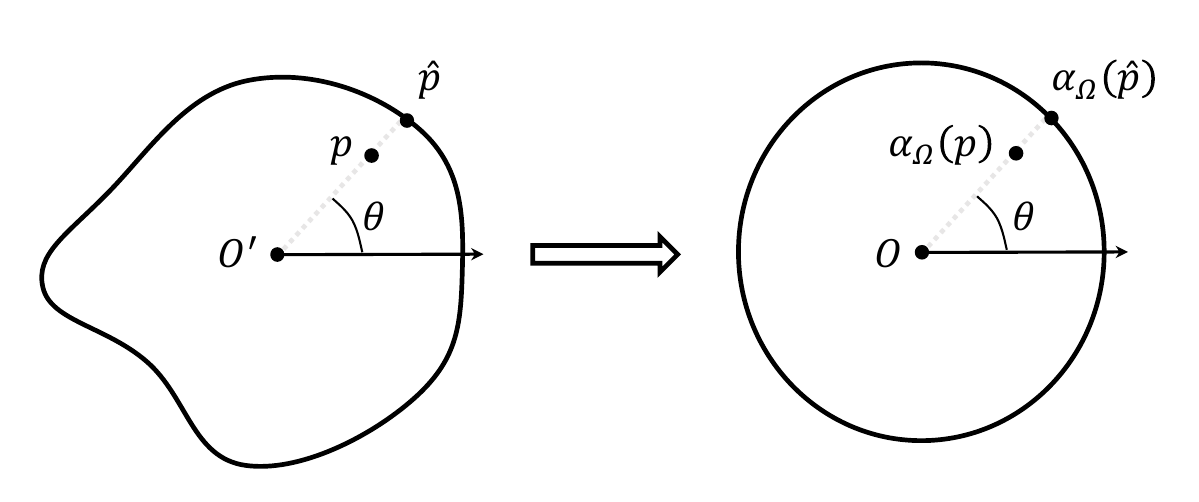}
\caption{The mapping $\alpha_{\Omega}$.}
\label{fig:transformation}
\end{figure}

With the transformation $\alpha_{\Omega}$, we can map $X$ onto the Cartesian product of $U$ and $C(B(0,1))$ as
\begin{equation}
\begin{split}
    \sigma:  X=\bigsqcup_{\Omega\in U}C (\overline{\Omega}) & \rightarrow U\times C(B(0,1)) \\
     f_{\Omega} & \mapsto (\Omega, f_{\Omega}\circ \alpha_{\Omega}^{-1}).
\end{split}
\end{equation}
It is easy to verify that $\sigma$ is a bijection with an inverse mapping $\sigma^{-1}(\Omega, f) = f\circ \alpha_{\Omega}$, thus we expect $U\times C(B(0,1))$ to be a metric space. Following this, we first give a metric on $U$. Let $\Omega_1,\Omega_2\in U$, then we define 
\begin{equation}
    d_{U}(\Omega_1,\Omega_2) := d_{E}(O_{\Omega_{1}}, O_{\Omega_{2}}) + \sup_{\theta \in [0,2\pi]}\vert b_{\Omega_{1}}(\theta) -  b_{\Omega_{2}}(\theta)\vert 
\end{equation}
where $O_{\Omega}$ is the centroid of $\Omega$, $d_E$ is the Euclidean metric, and $b_{\Omega}(\theta)$ denotes the boundary function of $\partial\Omega$ under polar coordinates as defined in Definition \ref{def:polar_region}. Through the metric on $U$ and the mapping $\sigma$, we can obtain the metric on $X$. Assuming $f_{\Omega_1} , f_{\Omega_2}\in X$, we define 
\begin{equation}
    d_{X}(f_{\Omega_1}, f_{\Omega_2}) = d_{U}(\Omega_1, \Omega_2) + \norm{f_{\Omega_1}\circ\sigma_{\Omega_1}^{-1} - f_{\Omega_2}\circ\sigma_{\Omega_2}^{-1} }_{C(B(0,1))}.
\end{equation}
It is easy to verify that these two metrics are well-defined.

With these foundations in place, we present the targeted mapping to be learned. Assuming $K$ is a compact set in $X$, we have the following mapping diagram:
\begin{equation}
\begin{split}
    \mathcal{G}:  K & \longrightarrow  X \\
       \updownarrow& ~~~~~~ \updownarrow \\
       \tilde{\mathcal{G}}:  \sigma(K) & \longrightarrow  U\times C(B(0,1))
\end{split}
\label{mapG}
\end{equation}
where $\tilde{\mathcal{G}} = \sigma \circ \mathcal{G} \circ \sigma^{-1} \in C(\sigma(K), U \times C(B(0,1)))$. Suppose that $\mathcal{G}$ is the solution mapping of the Poisson equation as \eqref{eq:poisson_mapping}, then $\mathcal{G}$ keeps the domain unchanged, so that we can define another mapping $\hat{\mathcal{G}}$ based on $\tilde{\mathcal{G}} $ as
\begin{equation}
    \tilde{\mathcal{G}}(\Omega, f) = (\Omega, \hat{\mathcal{G}}(\Omega, f)),\quad\hat{\mathcal{G}}\in C(\sigma(K), C(B(0,1))).
\label{hatG}
\end{equation}
Here $\hat{\mathcal{G}}$ is a mapping defined on $\sigma(K)$. To use MIONet for the approximation of $\hat{\mathcal{G}}$, we need to extend it to a larger domain. Consider the following two projections:
\begin{equation}
\begin{split}
    \pi_{1}:  U \times C(B(0,1)) & \rightarrow U,\quad\quad \pi_{2}:  U \times C(B(0,1)) \rightarrow C(B(0,1)), \\
             (\Omega, f)  & \mapsto \Omega\quad\quad\quad\quad\quad\quad\quad\quad\  (\Omega, f)\mapsto f  \\
\end{split}
\end{equation}
and denote the region $\pi_{1}(\sigma(K))\times \pi_{2}(\sigma(K))$ as $\tilde{K}$. Since $\pi_{1}$ and $\pi_{2}$ are continuous mappings, $\tilde{K}\subset U\times C(B(0,1))$ is a compact set. To demonstrate that $\hat{\mathcal{G}}$ can be extended to $\tilde{K}$, we invoke Dugundji's theorem \cite{dugundji1951extension}, which establishes that any continuous mapping from a compact set in a metric space to a locally convex linear space can be extended to the entire metric space. As a result, we have extended $\hat{\mathcal{G}}$ to $\tilde{K}$, i.e.,
\begin{equation}
\begin{split}
    \hat{\mathcal{G}}:  \tilde{K}=\pi_{1}(\sigma(&K))\times \pi_{2}(\sigma(K)) \longrightarrow  C(B(0,1)). \\
       \cap& ~~~~~~~~~~~~~~ \cap \\
       U & ~~~~~~~~C(B(0,1))
\end{split}
\end{equation}

Since $C(B(0,1))$ is a Banach space with a Schauder basis, it naturally satisfies Assumption \ref{ass:projection}.  However, we still need to prove that $U$ also satisfies Assumption \ref{ass:projection}. Now we define a mapping $\phi_{n}$ on $U$. For any $\Omega \in U$, assuming $O'$ is the centroid of $\Omega$. Let $x_i$ be the unique intersection point in $\{O'+te_i|t\geq0\}\cap\partial\Omega$ for $e_i:=(\cos(\frac{2i\pi}{n}),\sin(\frac{2i\pi}{n}))$. The discretization mapping $\phi_{n}$ is then defined as
\begin{equation}
\phi_{n}(\Omega) = (x_{0}, \cdots, x_{n-1})\in \mathbb{R}^{2n},
\end{equation}
and subsequently the reconstruction mapping $\psi_{n}$ is defined as
\begin{equation}
\psi_n(x_{0}, \cdots, x_{n-1}) = \hat{\Omega}\in U,
\end{equation}
where $\hat{\Omega}$ is the polygon formed by the vertices $x_{i}$. An illustration is shown in Figure \ref{fig:projection}.

\begin{figure}[htbp]
\centering
\includegraphics[width=0.9\textwidth]{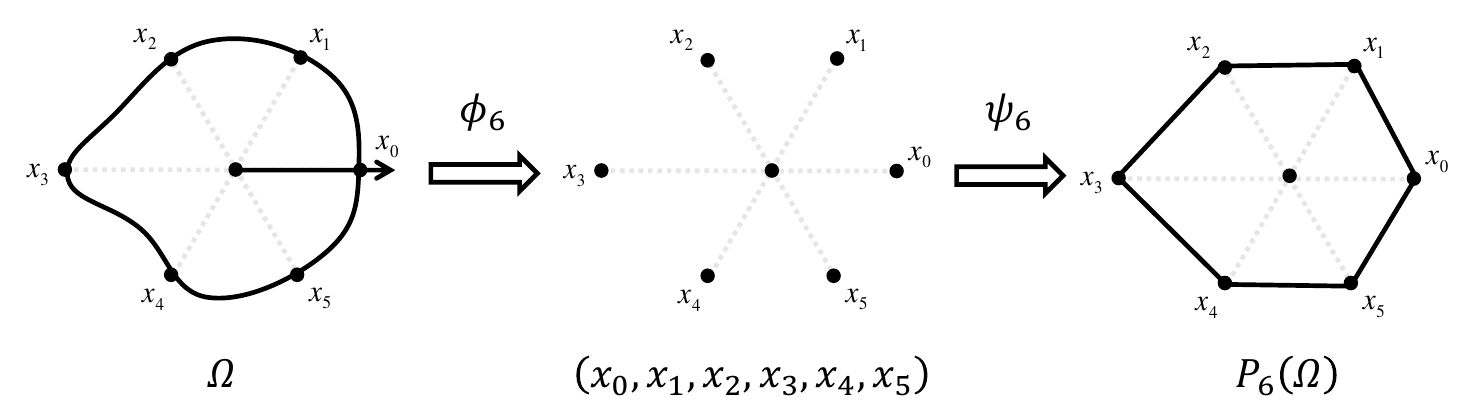}
\caption{An illustration of the discretization mapping and the reconstruction mapping.}
\label{fig:projection}
\end{figure}

Next, we prove that metric space $U$ with the mappings $\phi_{n}$ and $\psi_{n}$ satisfies Assumption \ref{ass:projection}.
\begin{theorem}\label{thm:assumption}
   The metric space $(U, d_{U})$ and the two sets of mappings $\{\phi_{n}\}$ and $\{\psi_{n}\}$ defined above satisfy Assumption \ref{ass:projection}, i.e.,
\begin{equation}
 \lim_{n\to\infty}\sup_{\Omega\in K}d_{U}(\Omega, P_{n}(\Omega)) = 0
\end{equation}
holds for any compact $K\subset U$, where $P_{n} = \psi_{n}\circ\phi_{n}$ is the projection mapping.
\end{theorem}
\begin{proof}
Denote the centroid and the area of $\Omega\subset \R^2$ as $C(\Omega)$ and $S(\Omega)$, respectively. For $\Omega\in K$, we have $S((\Omega\backslash P_n(\Omega))\cup(P_n(\Omega)\backslash\Omega))\leq\frac{C}{n}$ for a constant $C>0$, since the Lipschitz constant of $\partial\Omega$ has an upper bound, and the regions in $K$ are also bounded. It is not difficult to find that the area of $\Omega\in K$ has a positive lower bound, denoted by $S_0>0$ (choose a $\delta$ small enough for each $\Omega$ such that $S(\Omega')>\frac{1}{2}S(\Omega)$ for any $\Omega'\in B(\Omega,\delta)$, then consider the open covering), as well as an upper bound $S_1>S_0$. Hence $S(P_n(\Omega))\geq S_0-\frac{C}{n} \geq \frac{1}{2}S_0$ for $n$ large enough. Let $M$ be an upper bound of $|x|$ for $x\in\Omega\in K$, then
\begin{equation}
\begin{split}
d_{E}(C(\Omega),C(P_n(\Omega)))=&\norm{\frac{\int_\Omega xdx}{S(\Omega)}-\frac{\int_{P_n(\Omega)}xdx}{S(P_n(\Omega))}}_2 \\
=&\norm{\frac{(S(P_n(\Omega))-S(\Omega))\int_\Omega xdx+S(\Omega)(\int_\Omega xdx-\int_{P_n(\Omega)} xdx)}{S(\Omega)S(P_n(\Omega))}}_2 \\
\leq&\frac{CMS_1}{nS_0^2},
\end{split}
\end{equation}
for $n$ large enough. It immediately leads to
\begin{equation}\label{eq:centroid}
\lim_{n\to\infty}\sup_{\Omega\in K}d_{E}(C(\Omega),C(P_n(\Omega)))=0.
\end{equation}

Denote by $b_{\Omega}$ the boundary function of $\Omega$ as defined in Definition \ref{def:polar_region}. Let $e(\theta):=(\cos(\theta),\sin(\theta))$. Denote the intersection point in $\{C(\Omega)+t(b_{P_n(\Omega)}(\theta)e(\theta)+C(P_n(\Omega))-C(\Omega))|t\geq0\}\cap\partial\Omega$ as $Q$. Then
\begin{equation}
\begin{split}
|b_{\Omega}(\theta)-b_{P_n(\Omega)}(\theta)|=&\norm{b_{\Omega}(\theta)e(\theta)-b_{P_n(\Omega)}(\theta)e(\theta)}_2 \\
=&\big\|b_{\Omega}(\theta)e(\theta)+C(\Omega)-Q \\
&+Q-(b_{P_n(\Omega)}(\theta)e(\theta)+C(P_n(\Omega))) \\
&+b_{P_n(\Omega)}(\theta)e(\theta)+C(P_n(\Omega))-C(\Omega)-b_{P_n(\Omega)}(\theta)e(\theta)\big\|_2 \\
\leq&\norm{b_{\Omega}(\theta)e(\theta)+C(\Omega)-Q}_2+\frac{2\pi L}{n}+d_E(C(\Omega),C(P_n(\Omega))).
\end{split}
\end{equation}
Denote the angle between $Q-C(\Omega)$ and $b_{\Omega}(\theta)e(\theta)$ as $\alpha$, then $0\leq\alpha\leq\pi$ and
\begin{equation}
\begin{split}
\cos(\alpha)=&\frac{b_{P_n(\Omega)}(\theta)^2+\norm{b_{P_n(\Omega)}(\theta)e(\theta)+C(P_n(\Omega))-C(\Omega)}_2^2-d_E(C(\Omega),C(P_n(\Omega)))^2}{2b_{P_n(\Omega)}(\theta)\norm{b_{P_n(\Omega)}(\theta)e(\theta)+C(P_n(\Omega))-C(\Omega)}_2}\\
\geq&1-C_1d_E(C(\Omega),C(P_n(\Omega)))^2,
\end{split}
\end{equation}
where $C_1>0$ is a constant. The last inequality is due to $b_{P_n(\Omega)}(\theta)$ has a lower bound. Thus we have $\lim_{n\to\infty}\alpha=0$. Consequently,
\begin{equation}
\begin{split}
\norm{b_{\Omega}(\theta)e(\theta)+C(\Omega)-Q}_2^2=&b_{\Omega}(\theta)^2+\norm{Q-C(\Omega)}_2^2-2b_{\Omega}(\theta)\norm{Q-C(\Omega)}_2\cos(\alpha) \\
=&(b_{\Omega}(\theta)-\norm{Q-C(\Omega)}_2)^2+2b_{\Omega}(\theta)\norm{Q-C(\Omega)}_2(1-\cos(\alpha))\\
\leq&L^2\alpha^2+C_2d_E(C(\Omega),C(P_n(\Omega)))^2,
\end{split}
\end{equation}
for a constant $C_2>0$. Subsequently, we obtain
\begin{equation}
\begin{split}
|b_{\Omega}(\theta)-b_{P_n(\Omega)}(\theta)|\leq&\sqrt{L^2\alpha^2+C_2d_E(C(\Omega),C(P_n(\Omega)))^2}+\frac{2\pi L}{n}+d_E(C(\Omega),C(P_n(\Omega))),
\end{split}
\end{equation}
and then
\begin{equation}\label{eq:boundary}
\lim_{n\to\infty}\sup_{\Omega\in K}\norm{b_{\Omega}-b_{P_n(\Omega)}}_{C[0,2\pi]}=0.
\end{equation}
The equations \eqref{eq:centroid} and \eqref{eq:boundary} lead to the final result.
\end{proof}
\noindent This theorem implies that the mapping $\hat{\mathcal{G}}$ can be actually learned by MIONet. 

Up to now, we have completed the basic theory for problems defined on varying domains. We next summarize this method.

\subsection{Method}
Recall that the targeted mapping we need to learn is
\begin{equation}
\begin{split}
    \mathcal{G}:X=\bigsqcup_{\Omega\in U}C (\overline{\Omega}) & \rightarrow \bigsqcup_{\Omega\in U}C (\overline{\Omega}) \\
     f_{\Omega} & \mapsto u_{\Omega}.
\end{split}
\end{equation}
In the case of Poisson equation, the mapping $\mathcal{G}$ preserves the domains, so that $\mathcal{G}$ can be written as
\begin{equation}
\mathcal{G}=\sigma^{-1} \circ (\pi_1, \hat{\mathcal{G}}) \circ \sigma,
\end{equation}
where
\begin{equation}
\begin{split}
    \hat{\mathcal{G}}:=\pi_2\circ\sigma \circ \mathcal{G} \circ \sigma^{-1}\quad:\quad  &K_1\quad\times\quad K_2\quad \longrightarrow \quad C(B(0,1)), \\
       &\cap ~~~~~~~~~~~ \cap \\
       &U ~~~~~~~C(B(0,1))
\end{split}
\end{equation}
for a compact $K_1\subset U$ and a compact $K_2\subset C(B(0,1))$. Theorem \ref{thm:approximation} and Theorem \ref{thm:assumption} ensure that $\hat{\mathcal{G}}$ can be learned by MIONet. Assume that we have a dataset
\begin{equation}
\mathcal{T}=\{f_{\Omega_i}^i,u_{\Omega_i}^i\}_{i=1}^N,\quad \mathcal{G}(f_{\Omega_i}^i)=u_{\Omega_i}^i.
\end{equation}
We use a MIONet denoted by $\mathcal{M}$ to learn the corresponding $\hat{\mathcal{G}}$. The loss function can be written as
\begin{equation}
L(\theta)=\frac{1}{N}\sum_{i=1}^N\norm{\mathcal{M}(\sigma(f_{\Omega_i}^i);\theta)-\pi_2\circ\sigma(u_{\Omega_i}^i)}^2.
\end{equation}
After training, we predict a solution $u_\Omega$ for input $f_\Omega$ by
\begin{equation}
u_\Omega=\sigma^{-1} \circ (\pi_1, \mathcal{M}) \circ \sigma(f_{\Omega}).
\end{equation}
Note that $\sigma(f_{\Omega_i}^i)$ and $\pi_2\circ\sigma(u_{\Omega_i}^i)$ are preprocessed based on the dataset before training. An illustration of this method is shown in Figure \ref{fig:illustration}. The method of the fully-parameterized case can be found in Appendix \ref{app:full}.

\begin{figure}[htbp]
\centering
\includegraphics[width=1.0\textwidth]{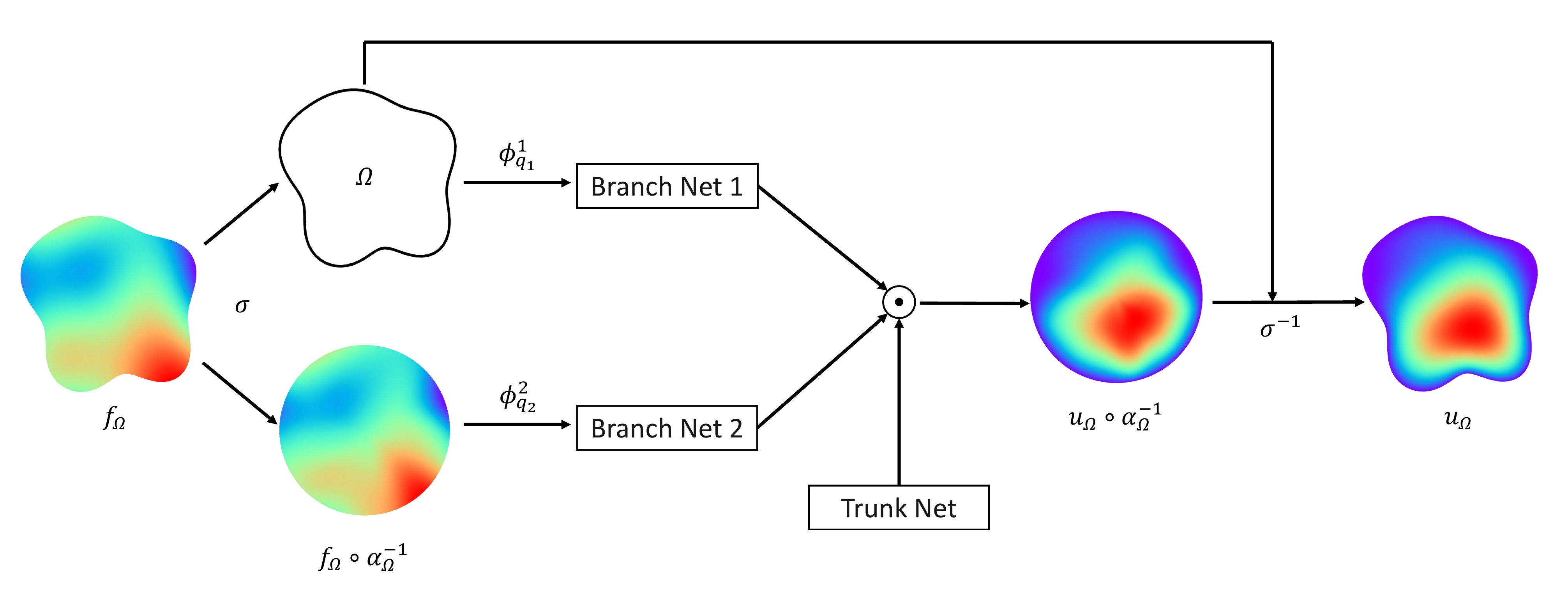}
\caption{An illustration of the method.}
\label{fig:illustration}
\end{figure}

\section{Numerical experiments}\label{sec:experiment}
We will validate the effectiveness of MIONet for PDEs defined on varying domains through several numerical experiments.

Consider the following Poisson equation:
\begin{equation}\label{eq:exp_poisson}
	\begin{cases}
		-\Delta u=f &\quad \mbox{in} \; \Omega,
		\\
		u = 0 &\quad \mbox{on} \; \partial\Omega.
	\end{cases}	
\end{equation}
Our objective is to learn the mapping 
\begin{equation}
    \mathcal{G}:f_{\Omega}\mapsto u_{\Omega},\quad f_{\Omega},u_{\Omega}\in C(\overline{\Omega}),
\end{equation}
which corresponding to
\begin{equation}
    \hat{\mathcal{G}}:(\Omega, f_{\Omega}\circ \alpha_{\Omega}^{-1})\mapsto u_{\Omega}\circ \alpha_{\Omega}^{-1},\quad f_{\Omega}\circ \alpha_{\Omega}^{-1},u_{\Omega}\circ \alpha_{\Omega}^{-1}\in C(B(0,1)).
\end{equation}

\begin{table}[htbp]
\centering
\begin{tabular}{|c|c|c|c|c|}
\hline
                  & Quadrilateral & Pentagon & Hexagon  & Smooth boundary  \\ \hline
$L^{2}$ error          & 1.50e-04      & 1.74e-04 & 2.00e-04 & 7.83e-05       \\ \hline
Relative $L^{2}$ error & 3.04e-02      & 2.80e-02 & 2.82e-02  & 3.00e-02      \\ \hline
\end{tabular}
\caption{$L^{2}$ errors and relative $L^{2}$ errors of different cases.}
\label{tab:errors}
\end{table}

The numerical result of the fully-parameterized case is shown in Appendix \ref{app:full}.

\subsection{Polygonal regions}

We first consider a simple case, in which the regions are convex polygons. We initially generate 1500 convex quadrilaterals (pentagons/hexagons) contained within $[0,1]^{2}$, totally 4500 regions. Then for each region, we generate a triangular mesh, thus we form the set of corresponding meshes. Following this, we generate 4500 random functions on $[0,1]^{2}$ via Gaussian Process (GP) with RBF kernel, forming the set of random functions for $f_{\Omega}$. We employ the finite element method to solve these 4500 equations of \eqref{eq:exp_poisson}, obtaining the set of solutions based on the meshes. We generate 5000 random but fixed points in $B(0,1)$ (the values of $f\in B(0,1)$ at these points are regarded as the coordinates of the truncated Schauder basis for $B(0,1)$), mapping them to the points in the polygons using $\alpha^{-1}_{\Omega}:B(0,1)\rightarrow \Omega$. Then we obtain the set of $f_{\Omega}\circ\alpha^{-1}_\Omega$ evaluated at such 5000 fixed points in $B(0,1)$. Finally, we use 200 points to encode the regions $\Omega$. Instead of using spatial coordinates, here we use 200 radii under polar coordinates to reduce the dimension. The dataset preprocessed for training can be written as
\begin{equation}
\mathcal{T}=\{(\phi_{200}^{1}(\Omega_i),\phi_{5000}^{2}(f_{\Omega_i}\circ\alpha^{-1}_{\Omega_i})),\phi_{5000}^{2}(u_{\Omega_i}\circ\alpha^{-1}_{\Omega_i})\}_{i=1}^{4500}.
\end{equation}

The architecture of our trained MIONet is as follows: the network comprises two branch nets. The first branch net encodes the information of input regions, i.e., a tensor of size (4500, 200). Its size is [200, 500, 500, 500, 500, 1000]. The second branch net deals with the information of input functions, i.e., a tensor of size (4500, 5000), and it has only one linear layer without bias as [5000, 1000], since $\hat{\mathcal{G}}$ is linear with respect to the second input. Additionally, the network includes a trunk net that encodes the output functions in $B(0,1)$, and its size is [2, 500, 500, 500, 500, 1000]. We use ReLU as the activation function, and compute the loss using MSE. During training, we employ the Adam \cite{kingma2014adam} optimizer with a learning rate of $10^{-6}$, and run the training for $5\times10^6$ iterations.

We employ the trained MIONet to predict the solutions of Poisson equations defined on arbitrarily convex $4, 5, 6$-polygons. The numerical results are shown in Table \ref{tab:errors}, and several examples are illustrated in Figure \ref{fig:polygon}. The prediction achieves nearly a $3\%$ relative $L^2$ error.

\begin{figure}[htbp]
\centering
\includegraphics[width=1.0\textwidth]{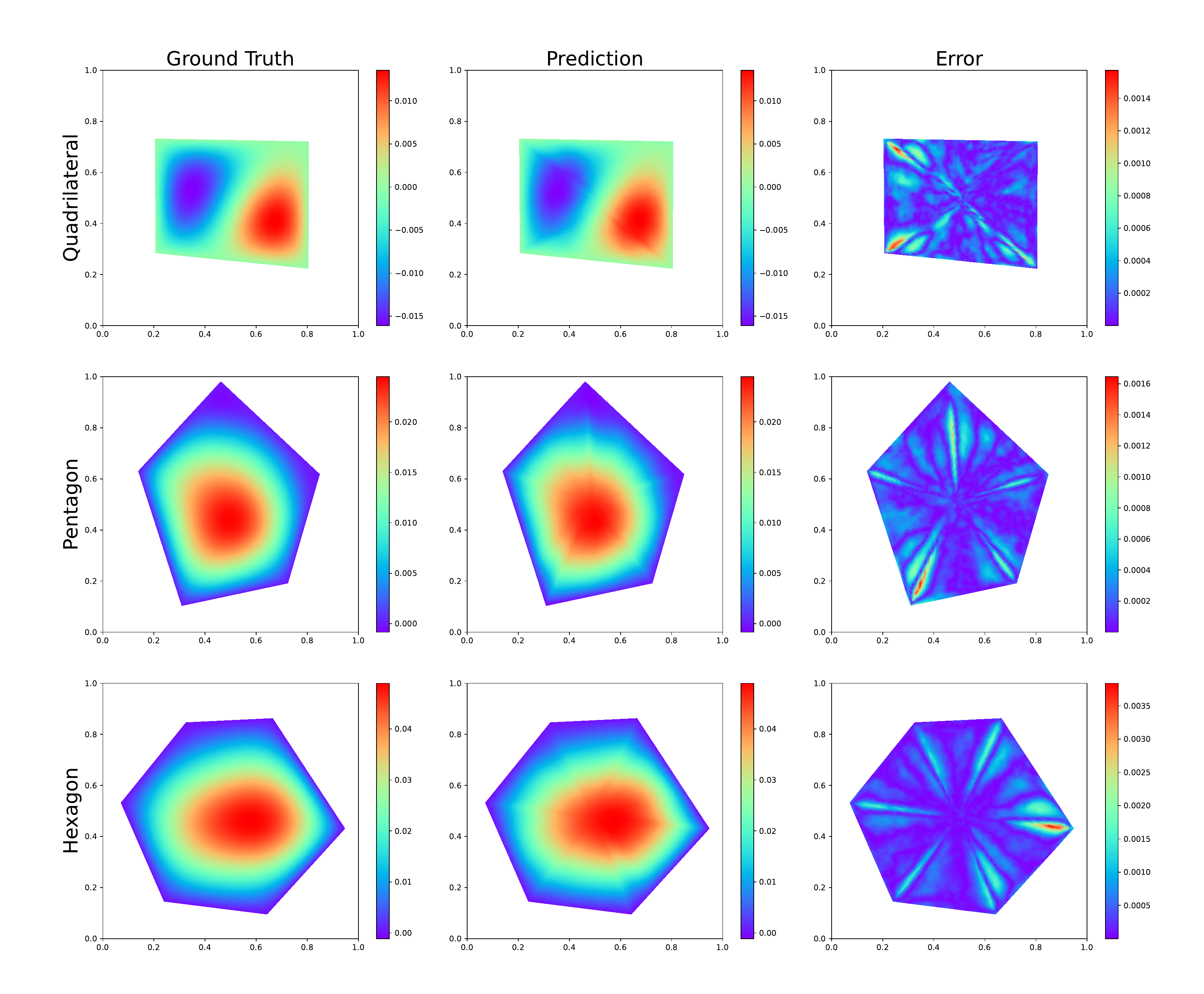}
\caption{Examples of predictions for Poisson equations defined on 4,5,6-polygons.}
\label{fig:polygon}
\end{figure}

\subsection{Polar regions with smooth boundary}

Now we consider the polar regions with smooth boundary. Note that the polar regions are unnecessary to be convex. First, we generate the required data. Initially, we use GP to generate 2500 random one-dimensional periodic smooth functions as the boundaries of regions. Similarly, we generate 2500 related meshes as well as 2500 random functions for $f_{\Omega}$, and then use the finite element method to solve these 2500 equations of \eqref{eq:exp_poisson}, obtaining the solutions of the Poisson equations on the meshes. Subsequently, we generate 5000 random but fixed points in $B(0,1)$ and map them to the generated polar regions using $\alpha^{-1}_{\Omega}$ to encode the original functions. We also choose 200 points to encode the regions as before. The dataset preprocessed for training can be written as
\begin{equation}
\mathcal{T}=\{(\phi_{200}^{1}(\Omega_i),\phi_{5000}^{2}(f_{\Omega_i}\circ\alpha^{-1}_{\Omega_i})),\phi_{5000}^{2}(u_{\Omega_i}\circ\alpha^{-1}_{\Omega_i})\}_{i=1}^{2500}.
\end{equation}
The architecture of the MIONet as well as the training parameters are the same as the previous experiment.

We present in Figure \ref{fig:smooth_boundary} the examples of using the model to predict the solutions of Poisson equations for three randomly generated polar regions and corresponding random functions of $f_\Omega$. The numerical results are also shown in Table \ref{tab:errors}. In this case MIONet also achieves a relative $L^2$ error of $3\%$, however, the predictive performance is superior to that for convex polygons, due to the favorable properties of the smooth boundaries.

\begin{figure}[htbp]
\centering
\includegraphics[width=1.0\textwidth]{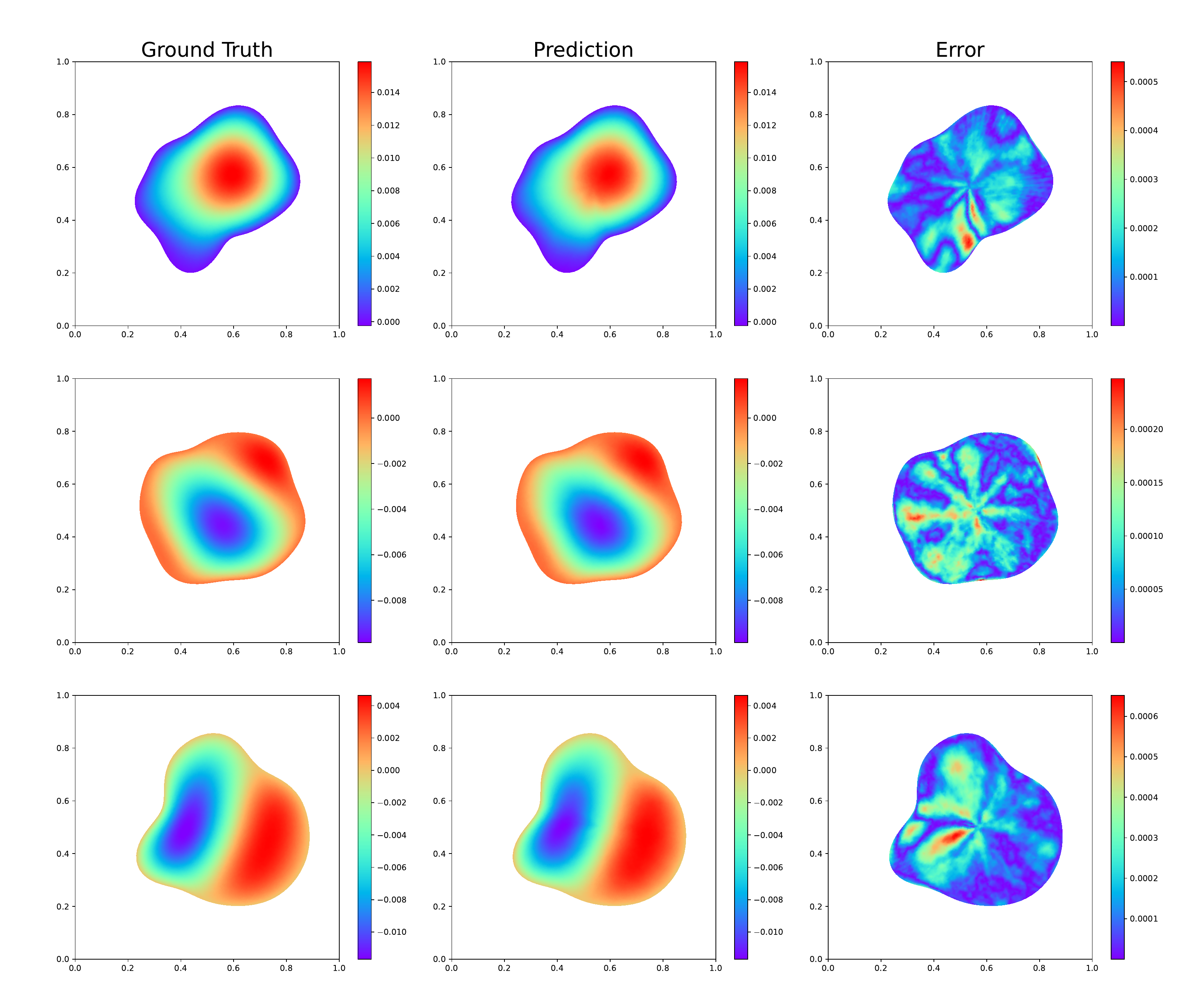}
\caption{Examples of predictions for Poisson equations defined on polar regions with smooth boundary.}
\label{fig:smooth_boundary}
\end{figure}

\subsection{Influence of mesh}

In engineering, the information of a given PDE problem is based on the mesh. Next we point out that the proposed method is meshless, and we test the influence of different meshes on one problem. Now assume that we have already trained a MIONet on a dataset, and $(f_{\Omega}, u_{\Omega})$ is a data point from the test set. We generate several meshes of different sizes for $\Omega$. Note that these meshes are also inconsistent on boundary, i.e., their nodes on $\partial\Omega$ are different.

Below, we conduct experiments on smooth polar regions with different boundary point samplings. In the previous experiment, for smooth boundaries generated by random functions, we sampled 100 points for discretization. In this section, we have chosen 100, 50, and 25 points, respectively, to discretize the boundary. Additionally, on these discretized polygons, meshes of sizes 0.01, 0.02, and 0.04 were employed, respectively. Finally, we utilized the trained MIONet to make predictions on these three regions with their corresponding meshes.

Our predictive results are presented in Table \ref{tab:mesh} and Figure \ref{fig:mesh}. The numerical results show that three different levels of discretization leads to similar errors. It can be observed that the predictive outcomes remain consistent even with variations in the sampling of smooth boundaries and meshing methods. This implies that our model is not significantly influenced by the discretization approach employed for the region.

\begin{table}[htbp]
\centering
\begin{tabular}{|c|c|c|c|}
\hline
    $\#$Boundary points             & 100    & 50     & 25     \\ \hline
$L^2$ error          & 4.24e-05  & 4.19e-05 &  4.47e-05 \\ \hline
Relative $L^2$ error & 4.59e-02  & 4.60e-02  & 5.12e-02  \\ \hline
\end{tabular}
\caption{$L^{2}$ errors and relative $L^{2}$ errors of different sizes of boundary points and meshes}
\label{tab:mesh}
\end{table}

\begin{figure}[htbp]
\centering
\includegraphics[width=1.0\textwidth]{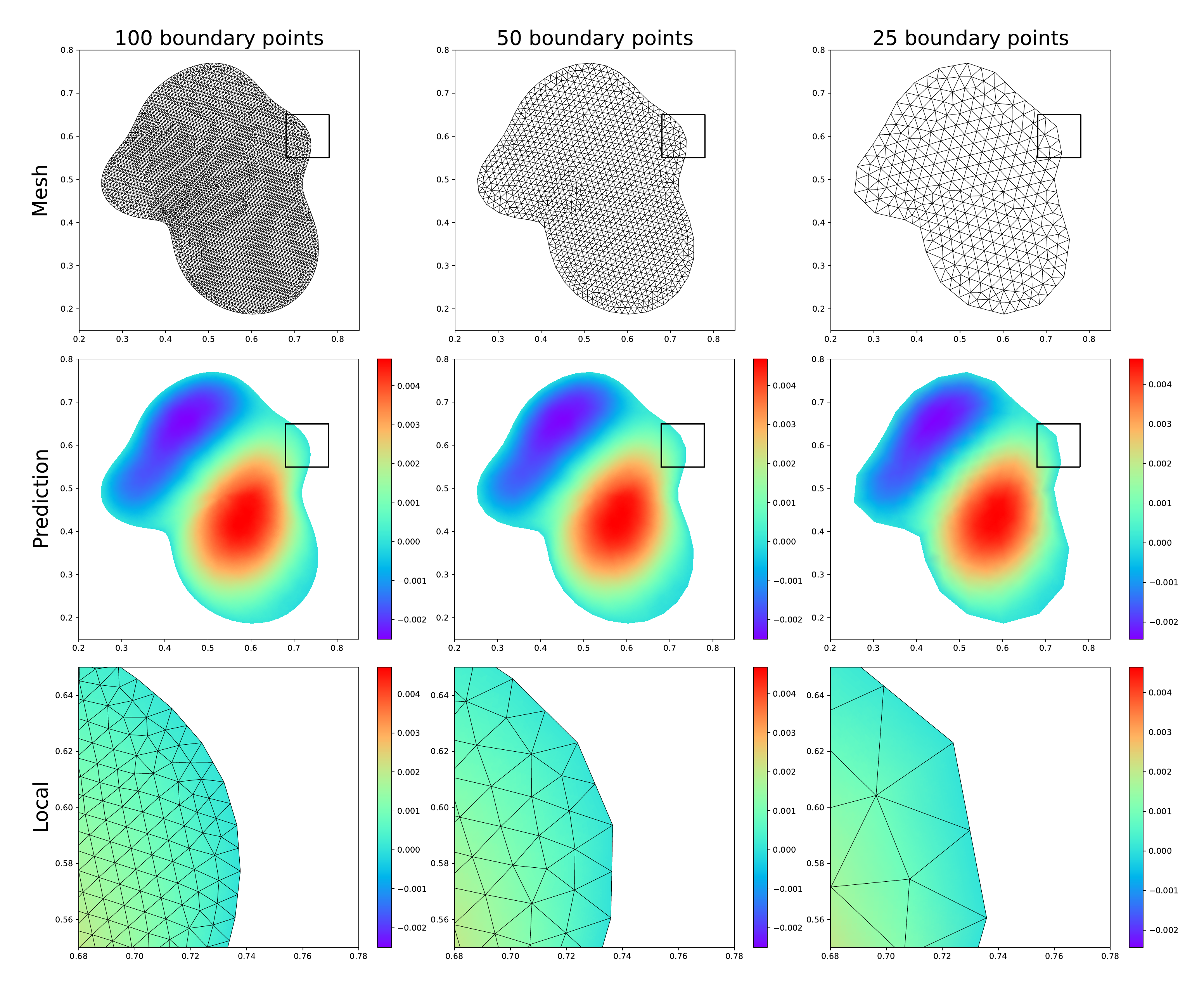}
\caption{Predictions for different levels of discretization on one task.}
\label{fig:mesh}
\end{figure}

\section{Conclusions}\label{sec:conclusions}

In this paper, we introduce a framework rooted in MIONet to address PDEs with varying domains. We first broaden the approximation theorem of MIONet to encompass metric spaces. Consequently, this extension enables its application to the metric space comprising an uncountable set of polar regions. We then present an algorithm tailored for PDEs with varying domains, hence we are able to learn the solution mapping of a PDE with all the parameters varying, including the parameters of the differential operator, the right-hand side term, the boundary condition, as well as the domain. We for example perform the experiments for 2-d Poisson equations, where the domains and the right-hand side terms are varying. The results provide insights into the method's performance across convex polygons, polar regions with smooth boundary, and predictions for different levels of discretization on one task. We also show the additional result of the fully-parameterized case in the appendix for interested readers.

As a meshless method, it is potentially used to train large models acting as general solvers for PDEs. Especially, we can employ this method as a tool for correcting low-frequency errors for the traditional numerical iterative solver, which has been studied as the hybrid iterative method \cite{hu2024hybrid}. In future works, we expect to further develop this method to deal with more complicated geometric regions, thus make such general solvers more powerful.

\section*{Acknowledgments}

This research is supported by National Natural Science Foundation of China (Grant Nos. 12171466 and 12271025).

\appendix

\section*{Appendix}

\section{Fully-parameterized solution operator}\label{app:full}
As we stated before, there is no difficulty in dealing with the fully-parameterized case as long as the simplified case of varying $f_\Omega$ is solved. Consequently we consider the fully-parameterized Poisson equation:
\begin{equation}\label{eq:full_poisson}
    \begin{cases}
		-\nabla\cdot(k\nabla u)=f &\quad \mbox{in} \; \Omega,
		\\
		u = g &\quad \mbox{on} \; \partial\Omega,
	\end{cases}
\end{equation}
where all the parameters including $k$, $f$, $g$ and $\Omega$ are changing, hence the solution mapping we aim to learn is
\begin{equation}
    \mathcal{G}:(k_{\Omega}, f_{\Omega}, g_{\partial\Omega}) \mapsto u_{\Omega}.
\end{equation}
Similarly, $\mathcal{G}$ can be rewritten as
\begin{equation}
\mathcal{G}=\sigma_2 \circ (\pi_1, \hat{\mathcal{G}}) \circ \sigma_1,
\end{equation}
where
\begin{equation}
    \begin{cases}
		\sigma_1:(k_{\Omega}, f_{\Omega}, g_{\partial\Omega})\mapsto \left(\Omega, k_\Omega\circ\alpha_{\Omega}^{-1},\left(f_\Omega\circ\alpha_{\Omega}^{-1},g_{\partial\Omega}\circ(\alpha_{\Omega}^{-1})|_{\partial B(0,1)}\right)\right), \\
		\sigma_2:(\Omega,u_\Omega\circ\alpha_\Omega^{-1})\mapsto u_\Omega, \\
		\pi_1:\left(\Omega, k_\Omega\circ\alpha_{\Omega}^{-1},\left(f_\Omega\circ\alpha_{\Omega}^{-1},g_{\partial\Omega}\circ(\alpha_{\Omega}^{-1})|_{\partial B(0,1)}\right)\right)\mapsto \Omega, \\
		\pi_2:(\Omega,u_\Omega\circ\alpha_\Omega^{-1})\mapsto u_\Omega\circ\alpha_\Omega^{-1},
	\end{cases}
\end{equation}
and
\begin{equation}
\begin{split}
	&\Omega ~~~~~~~~ k_\Omega\circ\alpha_{\Omega}^{-1}~~~~~~(f_\Omega\circ\alpha_{\Omega}^{-1},g_{\partial\Omega}\circ(\alpha_{\Omega}^{-1})|_{\partial B(0,1)})\mapsto u_\Omega\circ\alpha_\Omega^{-1} \\
    \hat{\mathcal{G}}:=\pi_2\circ\sigma_2^{-1}\circ\mathcal{G} \circ\sigma_1^{-1}:\quad  &K_1\quad\times\quad K_2\quad\quad\quad\quad\times\quad\quad\quad\quad K_3\quad\quad\longrightarrow\quad\quad C(B(0,1)). \\
       &\cap ~~~~~~~~~~ \cap ~~~~~~~~~~~~~~~~~~~~~~~~~~~~~ \cap\\
       &U ~~~~~~C(B(0,1))~~~~~~~~~~C(B(0,1))\times C(\partial B(0,1))
\end{split}
\end{equation}
Note that here we put the $f$ and $g$ together as a Cartesian product since the solution operator of Eq. \eqref{eq:full_poisson} is linear with respect to $(f,g)$. So that we use a MIONet to learn $\hat{\mathcal{G}}$, which has three branch nets. The first branch net encodes $\Omega$, the second encodes $k_\Omega\circ\alpha_{\Omega}^{-1}$, and the third encodes $(f_\Omega\circ\alpha_{\Omega}^{-1},g_{\partial\Omega}\circ(\alpha_{\Omega}^{-1})|_{\partial B(0,1)})$. Moreover, the third branch net is set to linear. With a trained MIONet $\mathcal{M}$, we make prediction given $(k_\Omega,f_\Omega,g_{\partial\Omega})$ by
\begin{equation}
u_{\Omega}^{\rm pred}=\sigma_2 \circ (\pi_1, \mathcal{M}) \circ \sigma_1(k_\Omega,f_\Omega,g_{\partial\Omega}).
\end{equation}

For experiment, we generate 4500 polar regions with smooth boundary, subsequently create 4500 corresponding triangular meshes based on these regions. Utilizing these meshes, we generate 4500 random functions $f_{\Omega}$, $k_{\Omega}$, and $g_{\partial\Omega}$ for the regions, and then solve equations \eqref{eq:full_poisson} employing the finite element method. Similar to the preceding experiments, we generate 5000 random but fixed points within $B(0,1)$ and employ $\alpha^{-1}_{\Omega}$ to establish mappings between $B(0,1)$ and $\Omega$. Additionally, we select 200 points to encode the regions and the boundary conditions. The training dataset is structured as follows:
\begin{equation}
\begin{split}
    \mathcal{T}=\{&[\phi_{200}^{1}(\Omega_i),\phi_{5000}^{2}(k_{\Omega_i}\circ\alpha^{-1}_{\Omega_i}),(\phi_{5000}^{2}(f_{\Omega_i}\circ\alpha^{-1}_{\Omega_i}), \phi_{200}^{3}(g_{\partial\Omega_i}\circ(\alpha^{-1}_{\Omega_i})|_{\partial B(0,1)}))], \\ &\phi_{5000}^{2}(u_{\Omega_i}\circ\alpha^{-1}_{\Omega_i})\}_{i=1}^{4500}.
\end{split}
\end{equation}
As for the architecture, the first branch network encodes information from the input regions, represented as a tensor of size (4500, 200), with neurons [200, 500, 500, 500, 500, 1000]. The second branch network processes information from the input functions $k$ as a tensor of size (4500, 5000), with neurons [5000, 500, 500, 500, 500, 1000]. The third branch network handles information from the input function $f$ and $g$, which is represented as a tensor of size (4500, 5200), featuring only one linear layer without bias, with neurons [5200, 1000]. Additionally, the network includes a trunk network responsible for encoding the output functions within $B(0,1)$, with neurons [2, 500, 500, 500, 500, 1000]. The training parameters remain consistent with those of the previous experiment.

The result of our model predicting fully-parameterized Poisson equation on the polar region is depicted in Figure \ref{fig:full}. The prediction achieves $4.30\times 10^{-4}$ $L^{2}$ error and 3.14$\%$ relative $L^{2}$ error.

\begin{figure}[htbp]
\centering
\includegraphics[width=1.0\textwidth]{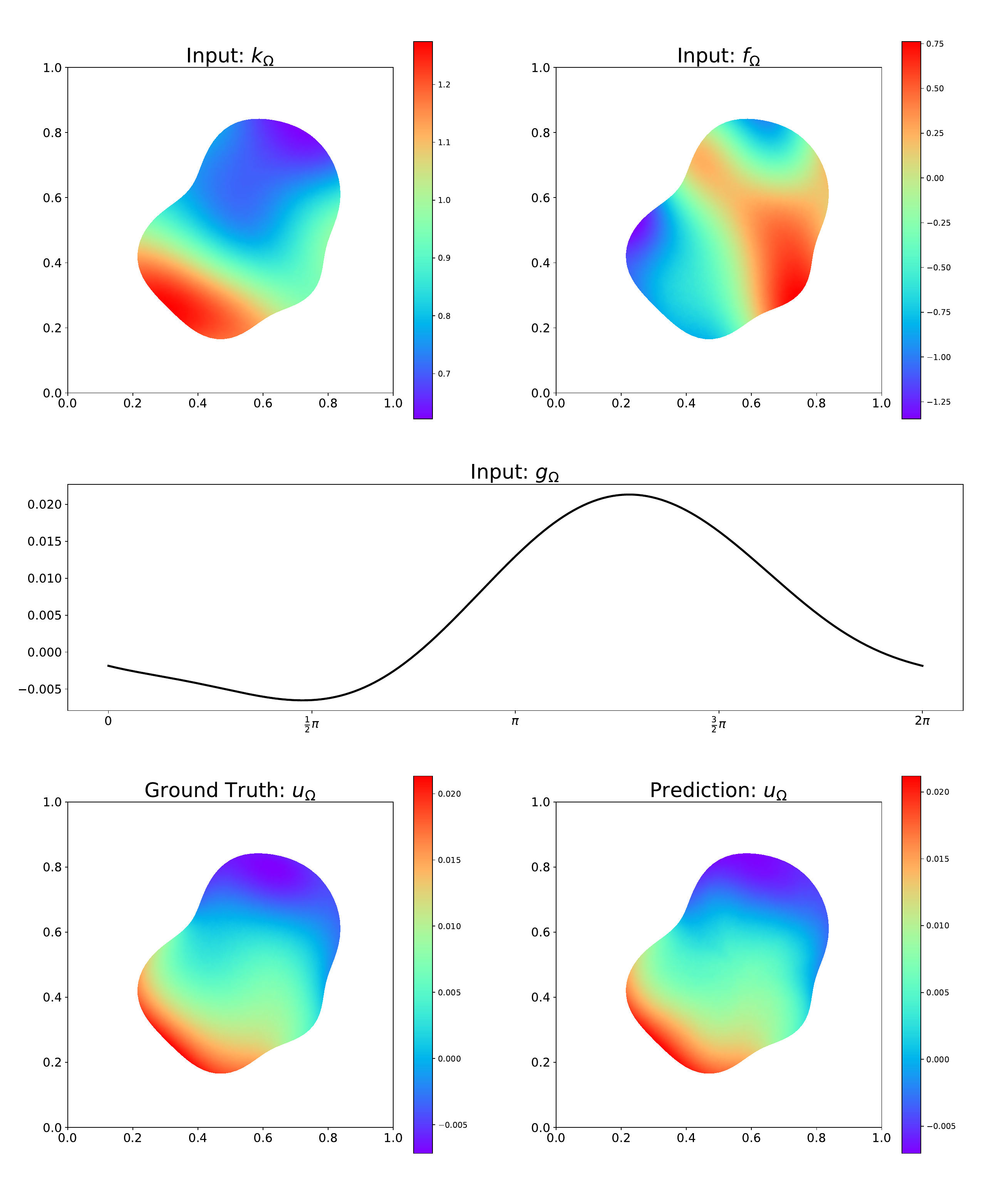}
\caption{An example of prediction for fully-parameterized Poisson equation.}
\label{fig:full}
\end{figure}

\bibliographystyle{abbrv}
\bibliography{ref}

\end{document}